\documentclass[11pt,reqno]{article}

\usepackage{booktabs} 
\usepackage{array} 
\usepackage{paralist} 
\usepackage{verbatim} 
\usepackage{amsthm}
\usepackage[table,xcdraw]{xcolor}
\usepackage[toc,title,page]{appendix}
\usepackage[symbol]{footmisc}

\usepackage{array}
\newcolumntype{C}[1]{>{\centering\arraybackslash$}p{#1}<{$}}

\usepackage{latexsym, amsmath, amssymb, a4, epsfig, color}
\usepackage{blindtext}
\usepackage{graphicx}
\usepackage{subfigure}
\usepackage[utf8]{inputenc}
\usepackage[export]{adjustbox}
\usepackage{wrapfig}
\usepackage{apptools}
\usepackage{siunitx}

\usepackage{floatrow}

\usepackage{amssymb}
\usepackage{amsmath}
\usepackage[normalem]{ulem} 

\usepackage{bm, mathtools}

\AtAppendix{\counterwithin{definition}{subsection}}
\AtAppendix{\counterwithin{theorem}{subsection}}

\graphicspath{}

\newtheorem{theorem}{Theorem}[section]
\newtheorem{corollary}[theorem]{Corollary}
\newtheorem{lemma}[theorem]{Lemma}
\newtheorem{prop}[theorem]{Proposition}
\newtheorem{definition}[theorem]{Definition}

\newtheorem{example}[theorem]{Example}
\newtheorem{remark}[theorem]{Remark}

\newcommand{\RR}{\mathbb{R}}

\newcommand{\NN}{\mathbb{N}}

\newcommand{\ZZ}{\mathbb{Z}}

\newcommand{\eqnref}[1]{(\ref {#1})}

\newcommand{\beq}{\begin{equation}}
\newcommand{\eeq}{\end{equation}}
\newcommand{\be}{\begin{equation*}}
\newcommand{\ee}{\end{equation*}}
\newcommand{\ba}{\begin{align*}}
\newcommand{\ea}{\end{align*}}
\newcommand{\bal}{\begin{align}}
\newcommand{\eal}{\end{align}}

\numberwithin{equation}{section}
\numberwithin{figure}{section}

\graphicspath{{.}{Figures/}}

\def\R{\mathbb{R}}
\def\L{\mathcal{L}}

\def\:={\coloneqq}

\def\<{\left\langle}
\def\>{\right\rangle}

\def\vecb{{\vec{b}}}
\def\vecB{{\vec{B}}}
\def\vecw{{\vec{w}}}
\def\vecW{{\vec{W}}}

\def\vecone{{\bf{1}}}
\def\boldp{{\mathbf{p}}}

\begin{document}

\title{
Provable wavelet-based neural approximation}

\author{Youngmi Hur\thanks{\footnotesize Department of Mathematics, Yonsei University, Seoul 03722, Republic of Korea (yhur@yonsei.ac.kr)} \and Hyojae Lim\thanks{Johann Radon Institute for
Computational and Applied Mathematics (RICAM), 4040 Linz, Austria \footnotesize (hyojae.lim@oeaw.ac.at)} 
\and Mikyoung Lim\thanks{\footnotesize Department of Mathematical Sciences, Korea Advanced Institute of Science and Technology, Daejeon 34141, Republic of Korea (mklim@kaist.ac.kr).}
   }

\date{\today}
\maketitle
\begin{abstract}

In this paper, we develop a wavelet-based theoretical framework for analyzing the universal approximation capabilities of neural networks over a wide range of activation functions. 
Leveraging wavelet frame theory on the spaces of homogeneous type, we derive sufficient conditions on activation functions to ensure
that the associated neural network 
approximates any functions in the given space, along with an error estimate. These sufficient conditions accommodate a variety of smooth activation functions, including those that exhibit oscillatory behavior. Furthermore, by considering the \(L^2\)-distance between smooth and non-smooth activation functions, we establish a generalized approximation result that is applicable to non-smooth activations, with the error explicitly controlled by this distance. This provides increased flexibility in the design of network architectures.
\end{abstract}

%
%


\section{Introduction}

Neural networks have long been recognized for their remarkable ability to approximate a wide range of functions, enabling state-of-the-art achievements across various fields in machine learning and artificial intelligence, image processing, natural language processing, and scientific computing (see, for example, \cite{Goodfellow:2016:DL,LeCun:2015:DL} and references therein). Various activation functions, such as ReLU, Sigmoid, Tanh, and oscillatory functions, have also been explored to further enhance network performance and adaptability. 

The versatility of neural networks originates from the structural flexibility of architectures that combine affine transformations with nonlinear activation functions. In addition, classical universal approximation theorems \cite{Cybenko:1989:ASS, Funahashi:1989:ARC, Hornik:1989:MFN} provide a theoretical basis for this flexibility by guaranteeing that, under suitable conditions, neural networks can approximate any continuous function on a bounded domain, underscoring their representational power.
These seminal results have been extended along various directions, including radial basis function (RBF) networks \cite{Li:2000:ARB, Park:1991:UAR}, non-polynomial activations \cite{Leshno:1993:MFN}, approximation of functions and their derivatives \cite{Hornik:1991:ACM, Li:1996:SAM}, the influence of network depth \cite{Eldan:2016:PDF}, approximation error bounds \cite{Barron:1994:AEB}, convolutional neural networks (CNN) \cite{Zhou:2020:UDC}, recurrent neural networks (RNN)  \cite{Schafer:2007:RNN}. 

As neural network architectures continue to evolve and diversify in practice, their theoretical foundations--beyond those provided by classical approximation theorems--have attracted increased attention. A particularly important challenge is to develop rigorous convergence analysis that accounts for a network's depth, size, and other architectural parameters. In this paper, we approach this problem via wavelet theory, specifically focusing on wavelet frame theory. 

Wavelet theory has proven to be a powerful tool for representing data or functions via superpositions of wavelet functions generated through translation and dilation. This structure provides a multi-resolution capability, making it possible to capture both local and global features. Moreover, rigorous convergence results have been established for functions in $L^2$ space, underlining the reliability of wavelet-based approximations. Building on these advantages, wavelet-based methods have found extensive application in the design and analysis of neural networks, including the development of architectures containing layers with wavelet activations~\cite{Huang:2005:NLA,Liu:2021:RIW,Liu:2019:MWC,Silva:2020:MLW,Sonoda:2021:RRO,Zhang:1992:WN}. 

Crucially, wavelet theory also offers a promising framework for achieving a provable understanding of the neural approximation. By leveraging wavelet frame theory on spaces of homogeneous type \cite{Deng:2008:HAS}, researchers have established convergence analyses that bound the approximation error in terms of the number of network nodes up to constant multiplication. This approach has been applied in particular to models using ReLU activations \cite{Shaham:2018:PAP} as well as to networks employing positive second-order differentiable activation functions with a radial quadratic structure \cite{Frischauf:2024:QNN} (see the detailed condition in Lemma 4.5 therein), with further applications \cite{Frischauf:2024:CNN}.

A significant hurdle, however, remains in extending these results to encompass more general activation functions, such as piecewise-smooth functions beyond ReLU. An especially noteworthy direction is to cover oscillatory activation functions, which have proven highly effective for problems exhibiting oscillatory behavior, such as boundary value problems in partial differential equations (PDEs) \cite{Jo:2020:DNN,Sitzmann:2020:INR}. Dealing with piecewise-smooth activation functions, including ReLU, is not straightforward because the wavelet frame theory in \cite{Deng:2008:HAS} assumes the so-called double Lipschitz condition, which does not hold when the derivative of the function has jump discontinuities. Although \cite{Shaham:2018:PAP} addresses ReLU, it offers limited details on handling such jump discontinuities in the derivative. In our study, we rigorously extend the approaches in \cite{Frischauf:2024:QNN,Shaham:2018:PAP} by broadening the class of permissible activation functions in convergence analysis. Specifically, we generalize the conditions on these functions to include oscillatory functions and provide additional results to encompass piecewise-smooth functions, thereby expanding the scope of convergence analysis in neural networks.

The wavelet frames in our paper are derived from harmonic analysis on spaces of homogeneous type \cite{Deng:2008:HAS}, offering a systematic framework that goes beyond the standard \(L^2(\mathbb{R}^d)\), as we will detail in Section~\ref{subS:homogeneous}. Although alternatively one might adopt classical methods (\cite{Han:2003:CST,Ron:1997:ASL,Chui:2002:CST, Daubechies:2003:FMC}) to build wavelet systems in \(L^2(\mathbb{R}^d)\), these often involve stricter conditions on the generating functions and filter structures, making the construction more rigid. In contrast, the approach via spaces of homogeneous type remains both broader and more flexible, allowing for a wider range of (and even oscillatory) activation functions, which suits well with our goal of approximating functions using wavelet-based neural networks. 

We now shift our discussion to the main results of our study. In this paper, we focus on the wavelet-inspired neural network \(\Psi_{\text{W}\vec{B}}\) (see Definition~\ref{def:WB-Net}). Our main result establishes sufficient conditions on the activation function, denoted by \(\sigma\), to ensure the associated wavelet-based network approximates any functions in \(\mathcal{L}_1\), which is a subspace of \(L^2(\R^d)\) (refer to (\ref{def:L1})). We present the precise statement of this result below.

\begin{theorem}
    \label{corollary: smooth WNN approximation}
    Let \(\sigma:\mathbb{R}^d \to \mathbb{R}\) be a twice-differentiable function satisfying \(\int_{\mathbb{R}^d} \sigma(x)dx =1\), \(\sigma(-x)= \sigma(x)\),
    and a decaying condition~\eqref{eqn: activation condition for AK}.
    Then for every \(f\in\mathcal{L}_1\) and every~\(N\in\mathbb{N}\), there exists a parameter set 
    \begin{equation}
        \boldp = \big[\gamma_1, \cdots, \gamma_{2N};\, \alpha_1, \cdots, \alpha_{2N};\, \vec{\theta}_1, \cdots, \vec{\theta}_{2N}\big]
        \label{eqn: parameter set of ourNN}
    \end{equation}
    of the W$\vec{B}$-Net \(\Psi_{\text{W}\vec{B}}\) with the activation function \(\sigma\) such that 
    \begin{equation}
    \label{eqn: main theorem}
        \left\lVert \Psi_{\text{W}\vec{B}}\left[\boldp\right] - f \right\rVert_{L^2} \leq \lVert f \rVert_{\mathcal{L}_1}(N+1)^{-1/2}.
    \end{equation}
\end{theorem}

This wavelet-based framework for neural networks, which was initially proposed in \cite{Shaham:2018:PAP} (as in~\eqref{eqn: ACHA Sk}), uses the wavelet system presented in \cite{Deng:2008:HAS} (as in \eqref{eqn: DengHan wavelet}) defined as follows: for \(k\in\mathbb{Z}\) and \(x,b\in\mathbb{R}^d\),
\begin{align}
S_k({x}, b) & = 2^k \sigma(2^{k/d}({x} - b)),
\label{eqn: ACHA Sk}
\\
\psi_{k,b}(x) & = 2^{-k/2} \left( S_k({x}, b) - S_{k-1}({x}, b) \right),
\label{eqn: DengHan wavelet}
\end{align}
where \(\sigma\in L^2(\mathbb{R}^d)\) is the neural network activation function. This framework stands out for its favorable convergence property, as shown in \eqnref{eqn: main theorem}, derived from wavelet frame theory. The error estimate explicitly depends on the number of network nodes \(N\), up to a constant factor. This feature allows more refined control over network complexity, distinguishing it from classical results such as Theorem~\ref{thm: Cybenko}.

In Theorem~\ref{corollary: smooth WNN approximation}, we require the activation function to be twice differentiable with sufficiently decaying derivatives, as specified in \eqref{eqn: activation condition for AK}. Notably, these conditions encompass oscillatory functions, as illustrated in examples of Section~\ref{subsec:main}. We then relax the smoothness assumption by considering the \(L^2\)-distance between a smooth activation function \(\sigma\) and a non-smooth activation function \(\sigma^\dagger\). Building on this, we establish a universal approximation result in \(\L_1\) for the network \(\Psi_{\text{W}\vec{B}}\) employing the generalized non-smooth activation, along with an error estimate that depends on the distance between \(\sigma\) and \(\sigma^\dagger\); this result is formalized in Corollary~\ref{cor:main:convergence}. Finally, we propose a practical strategy to control this distance while preserving a coherent neural network structure, detailed in Theorem~\ref{theorem:final}. These results extend the theoretical foundation for convergence in wavelet-based neural network approximation to a broader and more practically relevant class of activation functions.

The rest of this paper is organized as follows. In Section \ref{sec:pre}, we introduce the class of neural networks under consideration and provide a brief overview of wavelet theory on spaces of homogeneous type. In Section \ref{sec:wavelet:approx}, we construct the wavelet system using averaging kernels defined by neural network activation functions, and derive our main convergence theorems. In Section~\ref{sec:general}, we generalize our approximation results to accommodate a broader range of activation functions. Finally, Section~\ref{sec:conclusion} concludes the paper with a brief discussion, and detailed comparisons of networks are provided in the appendices.

\section {Preliminary}\label{sec:pre}

\subsection{Neural Networks}

We use $d$ to denote the spatial dimension. For $x\in\RR^d$, we may write $\vec{x}$ to emphasize that it is a vector. Among the numerous network architectures proposed in the literature, let us begin with neural networks defined as follows:

\begin{definition}\label{def:Psip}
Let \(L\in \mathbb{N}\) with \(L \geq 2\), and let \(N_1, \cdots, N_L \in \mathbb{N}\). Set $N_0=d$. We define a neural network \(\Psi_{NN}: \mathbb{R}^d \to \mathbb{R}^{N_L}\) by
\begin{equation*}
\Psi_{NN}\left[\boldp\right](\vec{x}) := \Psi_{NN}[W_1, \cdots, W_L;\, \vecb_1, \cdots, \vecb_L](\vec{x}) = A_L\left(\sigma\left(A_{L-1}\left( \cdots \left(\sigma\left(A_1(\vec{x})\right)\right)\right)\right)\right), \quad \vec{x} \in \mathbb{R}^d,
 \end{equation*}
with a nonlinear activation function \(\sigma\) that is applied to each component of the vector, and affine maps \(A_l: \mathbb{R}^{N_{l-1}} \to \mathbb{R}^{N_l}\) given by
\begin{equation*}
 A_l(\vec{x}) = W_l \,\vec{x} + \vecb_l, \quad \vec{x}\in \mathbb{R}^{N_{l-1}},\quad l=1,\dots,L,
\end{equation*}
where $\boldp=[W_1, \cdots, W_L;\, \vecb_1, \cdots, \vecb_L]$ is the parameter set with $W_l\in \mathbb{R}^{N_l\times N_{l-1}}$ and $ \vecb_l\in\RR^{N_l}$. 
\end{definition}

In this formulation, $L$ denotes the number of layers (excluding the input layer), \(N_1, \cdots, N_{L-1}\) represent the dimensions of the \(L-1\) hidden layers, and \(N_L\) is the dimension of the output layer. 

For the case \(L = 2\) with \(N_2 = 1\) and \(N_1 = N\) for some \(N \in \mathbb{N}\), and parameters \(W_1 \in \mathbb{R}^{N \times d}\), \(W_2 \in \mathbb{R}^{1 \times N}\), \(\vecb_1 \in \mathbb{R}^N, \vecb_2 = 0  \in \R\), the network in Definition~\ref{def:Psip} reduces to a shallow architecture. For notational convenience, we set \(W = W_1^T\), \(\vec{\alpha} = W_2^T \) and \(\vec{\beta} = \vecb_1\). Under these notations, the usual shallow network, which we refer to as the \textit{vector-Weight scalar-Bias Neural Network ($\vec{W}$B-Net)}, is then defined as follows:

\begin{definition}
A (shallow) \textit{vector-Weight scalar-Bias Neural Network ($\vec{W}$B-Net)} is the neural network \(\Psi_{\vec{W}B}: \mathbb{R}^d \to \mathbb{R}\) defined by
\begin{equation}
\Psi_{\vec{W}B}\left[\boldp\right](\vec{x}):= \Psi_{NN}[W;\vec{\alpha}; \vec{\beta}](\vec{x}) = \sum_{n=1}^N \alpha_n\, \sigma\left(\vec{w}_n \cdot \vec{x} + \beta_n\right), \quad \vec{x}\in\mathbb{R}^d,
        \label{eqn: VW-SB Net}
 \end{equation}
with a nonlinear activation function \(\sigma=\sigma_{1\to 1}:\mathbb{R} \to \mathbb{R}\), where $\boldp=[W;\vec{\alpha};\vec{\beta}]$ is the parameter set given with \(W = \begin{bmatrix}
        \vec{w}_1 & \cdots & \vec{w}_N
    \end{bmatrix} \in \mathbb{R}^{d \times N}\), \(\vec{\alpha} = (\alpha_1, \cdots, \alpha_N) \in \mathbb{R}^N\), \(\vec{\beta} = (\beta_1, \cdots, \beta_N)\in\mathbb{R}^N\).
\end{definition}

This structure aligns with the conventional neural network setup, where each neuron in the hidden layer is associated with a vector-valued weight and a scalar bias.

Finally, we introduce an alternative architecture, which we refer to as the \textit{scalar-Weight vector-Bias Neural Network ($W\vecB$-Net)} and use throughout this paper. In this setting, each neuron has a scalar weight and a vector bias. This structure is inspired by wavelet systems, in which scalar dilation and vector translation correspond to weight and bias, respectively, resulting in a distinct architecture compared to the conventional neural network~$\Psi_{\vec{W}B}$. We present a connection between $W\vecB$-Net and $W\vecB$-Net under specific conditions in appendix~\ref{appx:connection between neural networks}.

\begin{definition}
A (shallow) \textit{scalar-Weight vector-Bias Neural Network (W$\vec{B}$-Net)} is the neural network \(\Psi_{\text{W}\vec{B}}: \mathbb{R}^d \to \mathbb{R}\) defined by
\begin{equation}
\Psi_{\text{W}\vec{B}} \left[\boldp\right] (\vec{x}) = \sum_{n=1}^N \alpha_n\,\sigma(\gamma_n \vec{x} + \vec{\theta}_n), \quad \vec{x}\in\mathbb{R}^d,
\label{eqn: SW-VB Net}
\end{equation}
with a nonlinear vector-to-scalar activation function \(\sigma=\sigma_{d\to1}:\RR^d \to \RR\), where \(\boldp = [\vec{\gamma};\vec{\alpha};\Theta]\) is the parameter set given with \(\vec{\gamma} = (\gamma_1, \cdots, \gamma_N) \in\mathbb{R}^{N}\), \(\vec{\alpha} = (\alpha_1, \cdots, \alpha_N) \in\mathbb{R}^N\), and \(\Theta = [
 \vec{\theta}_1 \, \cdots \, \vec{\theta}_N] \in \mathbb{R}^{d\times N}\).
 \label{def:WB-Net}
\end{definition}

We now present a classical result in neural network approximation. Before doing so, we recall the definition of a discriminatory function, which plays a crucial role in universal approximation theorems.

\begin{definition}
A function \(\sigma:\mathbb{R} \to \mathbb{R}\) is called discriminatory if, for a measure \(\mu\) on \([0,1]^d\),
\begin{equation*}
        \int_{[0,1]^d} \sigma(\vec{w} \cdot \vec{x} + \theta) d\mu(\vec{x}) = 0 \quad \text{for all } \vec{w} \in\mathbb{R}^d \text{ and }\theta \in\mathbb{R}
 \end{equation*}
implies that \(\mu \equiv 0\).
\end{definition}

Notably, every non-polynomial function is discriminatory \cite{Leshno:1993:MFN}. In particular, any bounded, measurable sigmoidal function satisfies this criterion. We now recall Cybenko's universal approximation theorem, which asserts that a $\vecW$B-Net with a continuous discriminatory activation function can approximate any continuous function on a compact domain arbitrarily well, as follows.

\begin{theorem}[\cite{Cybenko:1989:ASS}]
    \label{thm: Cybenko}
    Let \(\sigma:\mathbb{R}\to\mathbb{R}\) be a continuous discriminatory function. Then for every function \(f \in C([0,1]^d)\) and \(\epsilon > 0\), there exist \(N_\epsilon\in\mathbb{N}\) and a parameter set
    \begin{equation*}
        \boldp = \left[\vecw_1, \cdots, \vecw_{N_\epsilon};\,\alpha_1, \cdots, \alpha_{N_\epsilon};\,\beta_1, \cdots, \beta_{N_\epsilon}\right]
    \end{equation*}
    of the $\vecW$B-Net, \(\Psi_{\vecW B}\), in (\ref{eqn: VW-SB Net}) such that, 
    for all \(\vec{x} \in [0,1]^d\),
    \begin{equation*}
        \left\lvert f(\vec{x}) - \Psi_{\vecW B}\left[\boldp\right](\vec{x}) \right\rvert < \epsilon.
    \end{equation*}
\end{theorem}

Beyond Cybenko's result, many other universal approximation results have been established for a variety of network architectures (e.g.,  \cite{Frischauf:2024:QNN}) and function spaces (e.g.,  \cite{Hornik:1991:ACM}).

\subsection{Wavelet expansions on spaces of homogeneous type}
\label{subS:homogeneous}

There are various ways to construct a wavelet system that enables wavelet series expansion. In this paper, we achieve wavelet expansion by constructing a wavelet system on a space of homogeneous type. To begin, we first introduce the definition of \textit{a space of homogeneous type}.

\begin{definition}\rm{\cite[Definition 1.1]{Coifman:1971:AHN}}
    A \textit{space of homogeneous type} \((X,\mu, \delta)\) is a set \(X\) together with a measure \(\mu\) and a quasi-metric \(\delta\) (which satisfies triangle inequality up to a constant) such that for every \(x \in X\) and \(r>0\), 
    \begin{enumerate}[\rm(i)]
        \item \(0< \mu(B(x,r))< \infty\), and 
        \item there exists a constant \(C<\infty\) such that \(\mu(B(x,2r)) \leq C \mu(B(x,r))\).
    \end{enumerate}
    Here, $B(x,r)$ denotes the ball of radius $r$ centered at $x$ defined by the quasi-metric $\delta$.
\end{definition}

\vskip 5mm

We employ the space of homogeneous type \((X,\mu,\delta)\), where \(X=\RR^d\), \(\mu\) is the Lebesgue measure, and \(\delta\) is the Euclidean metric. 
Following \cite{Shaham:2018:PAP} (see also \cite{Deng:2008:HAS} for more details), we use the quasi-metric \(\rho(x,b) = c \lVert x-b \rVert^d\) for $x,b\in\RR^d$ with some constant\footnote{We reserve the letter \(c\) to denote this constant throughout the paper.} \(c>0\), which can be shown to induce the same topology as \(\delta\).
Then, we can set \(\theta=1/d\) and \(A=3^d/2\) in the following definition (refer to (1.3) and (1.7) in \cite{Deng:2008:HAS}). 
Here, we denote by $\|x\|$ the Euclidean norm of $x\in\RR^d$.

We now introduce a family of \textit{averaging kernels} and the associated wavelet system.

\begin{definition}
\rm{\cite[Definitions 3.4 and 3.5]{Deng:2008:HAS}} \label{def:averaging}
Let \((X,\mu,\delta)\) be a space of homogeneous type. A collection of symmetric functions~\(\{S_k\}_{k\in\mathbb{Z}}\), each \(S_k: X \times X \to \mathbb{C}\), is said to be a family of \textit{averaging kernels} if 
there exist $0 < \eta, \epsilon \leq \theta$ and $C < \infty$, independent of $k$, satisfying the following conditions: for all $x,x',y,y'\in X$,
\begin{equation}
    \int S_k(x, y)\, dy = 1;
    \tag{C1}
    \label{eqn: averaging kernel3}
\end{equation}
\begin{equation}
    \left|S_k(x, y)\right| \leq C\, \frac{2^{-k\epsilon}}{\left(2^{-k} + \rho(x, y)\right)^{1+\epsilon}};
    \tag{C2}
    \label{eqn: averaging kernel1}
\end{equation}
\begin{equation}
    \left|S_k(x, y) - S_k(x', y)\right| \leq C \left( \frac{\rho(x, x')}{2^{-k} + \rho(x, y)} \right)^{\eta} 
    \frac{2^{-k\epsilon}}{\left(2^{-k} + \rho(x, y)\right)^{1+\epsilon}}
    \tag{C3}
    \label{eqn: averaging kernel2}
\end{equation}
if $\rho(x, x') \leq \frac{1}{2A}\left(2^{-k} + \rho(x, y)\right)$;
\begin{equation}
\begin{aligned}
&\left|S_k(x, y) - S_k(x', y) - S_k(x, y') + S_k(x', y')\right| \\
  \leq &\, C \left( \frac{\rho(x, x')}{2^{-k} + \rho(x, y)} \right)^\eta \left( \frac{\rho(y, y')}{2^{-k} + \rho(x, y)} \right)^\eta \frac{2^{-k \epsilon}}{\left(2^{-k} + \rho(x, y)\right)^{1 + \epsilon}}
\end{aligned}
\tag{C4}
\label{eqn: averaging kernel4}
\end{equation}
if $\rho(x, x') \leq \frac{1}{2A}\left(2^{-k} + \rho(x, y)\right)$ and $\rho(y, y') \leq \frac{1}{2A}\left(2^{-k} + \rho(x, y)\right)$.
\end{definition}

Here, `$S_k$ being symmetric' means that $S_k(x,y)=S_k(y,x)$ for all $x,y\in X$. In \cite{Deng:2008:HAS}, this symmetry assumption is not imposed to define averaging kernels. Instead, those kernels are introduced under additional conditions--analogous to (\ref{eqn: averaging kernel3}) and (\ref{eqn: averaging kernel2})--by interchanging the roles of $x$ and~$y$. For simplicity, we assume this symmetry condition and focus on the reduced set of conditions. Also, we note that the condition \eqnref{eqn: averaging kernel4} is called the \textit{double Lipschitz condition}. 

\begin{definition}\rm{\cite[Definition 3.14]{Deng:2008:HAS}}
Let $\{S_k\}_{k\in\mathbb{Z}}$ be a family of averaging kernels on $X\times X$. For each $k\in\mathbb{Z}$, define
\begin{equation*}
    D_k(x,b) := S_k(x,b) - S_{k-1}(x,b), 
    \quad D_k: X \times X \to \mathbb{C}.
\end{equation*}
Then, for $b \in X$, we set
\[
    \psi_{k,b}(x) := 2^{-k/2} \, D_k(x,b).
\]
The family $\{\psi_{k,b}\}$ is said to be a \textit{wavelet system} (associated with the averaging kernels $\{S_k\}$).
\label{def: wavelet system}
\end{definition}

The countable subset \(X(\psi) := \{\psi_{k,b}\}_{(k,b)\in\Lambda}\) of the above wavelet system associated with the averaging kernels $\{S_k\}$ provides the following wavelet series expansion in $L^2(\RR^d)$, along with its counterpart \(X(\widetilde{\psi}) = \{\widetilde{\psi}_{k,b}\}_{(k,b)\in\Lambda}\), where $\Lambda$ is the discrete index set specified in the theorem.

\begin{theorem}\rm(\cite[Theorem 3.25]{Deng:2008:HAS}).
   Let \(\{S_k\}_{k\in\mathbb{Z}}\) be a family of averaging kernels, and let \(X(\psi)= \{\psi_{k,b}\}_{(k,b)\in\Lambda}\) denote the discrete wavelet system derived from these kernels. Then there exists a discrete wavelet system \(X(\widetilde{\psi})=\{\widetilde{\psi}_{k,b}\}_{(k,b)\in\Lambda}\) such that, for all~\(f\in~L^2(\mathbb{R}^d)\),
    \begin{equation}
        f = \sum_{(k,b)\in \Lambda} \langle f, \widetilde{\psi}_{k,b} \rangle \psi_{k,b}.
        \label{eqn: frame}
    \end{equation}
Here, \(\Lambda = \{(k,b) \in \mathbb{Z} \times \mathbb{R}^d: b \in 2^{-k/d} \mathbb{Z}^d\}\) and \(\langle \cdot, \cdot \rangle\) denotes the usual inner-product in \(L^2(\R^d)\).
    \label{thm: wavelet frame approximation in homogeneous type space}
\end{theorem}

\begin{remark}
    Note that \(\{\psi_{k,b}\}_{(k,b)\in\Lambda}\) and its dual \(\{\widetilde{\psi}_{k,b}\}_{(k,b)\in\Lambda}\) are both wavelet frames.
\end{remark}

We now introduce the space $\mathcal{L}_1$ via the wavelet frame \(\{\psi_{k,b}\}_{(k,b)\in \Lambda}\), following the terminology of \cite{Barron:2008:ALG,Shaham:2018:PAP}, as
\begin{equation}\label{def:L1}
\mathcal{L}_1 \;=\; \left\{\,f \in L^2(\mathbb{R}^d)\,\colon\,  \left\|f\right\|_{\mathcal{L}_1} <\infty
 \right\}
 \end{equation}
 with
\[ \left\|f\right\|_{\mathcal{L}_1} :=\;
\inf \left\{\,{\textstyle\sum_{(k,b)\in\Lambda}}\, \left|c_{k,b}\right|: f = {\textstyle\sum_{(k,b)\in\Lambda}} \, c_{k,b}\,\psi_{k,b}\right\}.
\]
In other words, \(\mathcal{L}_1\) consists of those \(L^2\)-functions having an absolutely
summable expansion in the wavelet frame. Let \(f \in \mathcal{L}_1\), and
suppose we approximate \(f\) by repeatedly selecting the frame element yielding
the largest inner product with the current residual, orthogonalizing at each step
step; this procedure is known as the \textit{orthogonal greedy algorithm (OGA)}.
A classical result (cf.\ \cite[Theorem~2.1]{Barron:2008:ALG} and also
\cite[Section~3.1]{Shaham:2018:PAP}) states that the greedy approximant \(f_N\) obtained
after $N$ steps (and hence is a linear combination of at most \(N\) elements in the wavelet frame) satisfies an \(L^2\) error bound of order \((N+1)^{-1/2}\).

\begin{theorem}[\cite{Barron:2008:ALG,Shaham:2018:PAP}]
Let \(f \in \mathcal{L}_1\) and \(\{f_N\}\) be the sequence of greedy approximants
produced by OGA from the wavelet frame \(\{\psi_{k,b}\}_{(k,b)\in\Lambda}\). Then
\begin{equation}\label{ineq:f_N:OGA}
\left\|f - f_N\right\|_{L^2}
\;\leq\;
\left\|f\right\|_{\mathcal{L}_1}\,(N+1)^{-1/2}, \quad\mbox{for each }N\in\NN.
\end{equation}
    \label{thm: OGA}
\end{theorem}

At this point, we recall the definition of W$\vec{B}$-Net with an activation function \(\sigma\) and a parameter set~\(\boldp\)~(see \eqref{eqn: SW-VB Net}). Under the settings of the following section (Section~\ref{sec:wavelet:approx}), we can interpret the approximation function $f_N$ as a W$\vec{B}$-Net of $2N$ terms:
$$
\Psi_{\text{W}\vec{B}} \left[\boldp\right](\vec{x}) = \sum_{n=1}^{2N} \alpha_n\,\sigma(\gamma_n \vec{x} + \vec{\theta}_n), \quad \vec{x}\in\mathbb{R}^d,
$$
where $\alpha_n,\gamma_n,\vec{\theta}_n$ are learnable parameters. The parameter $N$ is related to the number of nodes in the neural network; see the end of Section \ref{subsec:main} for further details.

\section{Wavelet-based neural approximation}\label{sec:wavelet:approx}

In this section, we develop a neural approximation based on the wavelet frame theory introduced in Section \ref{subS:homogeneous}. To do so, we connect the kernels \(\{S_k\}\) in Definition \ref{def:averaging}, which are used to form the wavelet system \(\{\psi_{k,b}\}\), with the neural network's activation function~\(\sigma\). The following definition establishes this link. 
This approach of defining the kernel \(S_k\) in terms of the activation function~\(\sigma\) is first introduced in \cite{Shaham:2018:PAP}, where \(\sigma\) is specifically chosen as a multi-layer composition of linear combinations of ReLU functions. Additionally, in \cite{Frischauf:2024:QNN}, \(\sigma\) is formulated as a radial quadratic function. 

In the present paper, we further extend these ideas to identify more general activation functions for W$\vecB$-Net by providing the sufficient conditions on~\(\sigma\) under which \(\{\psi_{k,b}\}\) forms a wavelet frame.

\begin{definition}\label{def:Sk and psi_kb}
Let \(\sigma \in L^2(\mathbb{R}^d)\).
For \( k \in \mathbb{Z} \), we define
\begin{align}
&S_k({x}, b) := 2^k \sigma(2^{k/d}({x} - b)), \quad x,b\in\mathbb{R}^d,\label{def:Sk} \\
&\psi_{k,b}(x):= 2^{-k/2}D_k({x}, b) = 2^{-k/2} \left( S_k({x}, b) - S_{k-1}({x}, b) \right), \quad x,b\in\mathbb{R}^d.\label{def:psi_kb}
\end{align}
\end{definition}

For convenience, we continue using the notation $\psi_{k,b}$, even when \(\{S_k\}\) does \textit{not} form a family of averaging kernels, in which case the collection \(\{\psi_{k,b}\}\) may \textit{not} a wavelet system. In Theorems~\ref{corollary: smooth WNN approximation} and~\ref{thm:wave:error}, we use \(\psi_{k,b}\) to denote the wavelet system under which \(\{S_k\}\) does form a family of averaging kernels, aligning with Definition~\ref{def: wavelet system}. In contrast, for other results such as Theorem~\ref{lem:disconti}, we do not require the conditions in Definition~\ref{def:averaging}, and thus \(\{S_k\}\) and \(\{\psi_{k,b}\}\) there need \textit{not} be averaging kernels or a wavelet system, respectively. Nevertheless, we maintain the same notation throughout for simplicity.

\subsection{Main results}\label{subsec:main}

In this subsection, we present our main results. We first establish sufficient conditions for, possibly sign-changing, activation functions~$\sigma$ that ensure their associated kernels $\{S_k\}$ to satisfy the conditions in Definition~\ref{def:averaging}. We will then apply the wavelet frame theory.

\begin{prop}\label{thm:averaging}
    Let \(\sigma:\mathbb{R}^d \to \mathbb{R}\) be a twice-differentiable function satisfying \(\int_{\mathbb{R}^d} \sigma(x)dx =1\) and, for every \(x\in\R^d\), \(\sigma(-x)= \sigma(x)\) and 
    \begin{equation}
        \lVert \nabla_x^j \sigma(x) \rVert \leq \frac{C'}{\left(c^{-1} + \lVert x \rVert^d\right)^{1+\epsilon+j/d}}, \quad j=0,1,2
        \label{eqn: activation condition for AK}
    \end{equation}
    with some constant \(C'>0\).
    Then \(\{S_k\}_{k\in\mathbb{Z}}\) defined by (\ref{def:Sk}) is a family of averaging kernels. 
\end{prop}

We defer the proof of the proposition to Section~\ref{subsec: proof}. We closely follow the steps of the proof in  \cite{Frischauf:2024:QNN} and \cite{Shaham:2018:PAP}. Below are examples of activation functions that comply with the \textit{averaging kernel conditions}, meaning they meet all the assumptions in Proposition~\ref{thm:averaging}.

\begin{example}
\label{ex: smooth activation}
The following functions $\sigma$ satisfy the averaging kernel conditions. Here, $m$ is a real constant, and \(C\) and \(C_m\) are normalizing constants chosen so that \( \int_{\mathbb{R}^d} \sigma(x) \, dx = 1 \).
\begin{itemize}
\item Let
$$\sigma(x)=C_m\, \widetilde{\sigma}(x)\sin(m x), \quad x\in\RR,$$
where  $\widetilde{\sigma}$ is an odd function with respect to $x$ and is defined as
\begin{equation*}
    \widetilde{\sigma}(x) = \begin{cases}
    \widetilde{\sigma}_0(x), & \lvert x \rvert \leq 1,\\
       1/x^{\alpha}, & \lvert x \rvert >1,
    \end{cases}
\end{equation*}
for some bounded function \(\widetilde{\sigma}_0\) that is smoothly connected at \(\lvert x \rvert = 1\) in such a way that \(\widetilde{\sigma}\) is twice differentiable. Here, $\alpha$ is a constant \(> 3\).

\item Using the positive-valued activation functions $\widetilde{\sigma}$ from \cite{Frischauf:2024:QNN} (or their generalizations), let
$$\sigma({x})=
C \,\widetilde{\sigma}\left(r^2- \lVert {x} \rVert^2\right)\cos({\tau}\cdot {x}), \,\,
C_m \,\widetilde{\sigma}(r^2-\|{x}\|^2)\,\frac{\sin(m\|{x}\|^2)}{\|{x}\|^{2}},\quad {x}\in\RR^d.$$
Here, \({\tau}\) is a constant vector in $\RR^d$.
\end{itemize}
\end{example}

\begin{example}
    We continue to use the same notation \(C\) as in the previous example. Let us consider activation functions \(\sigma\) of the form (decaying function)*(oscillatory function). Our goal is to determine sufficient conditions on the oscillatory component, assuming that the decaying component satisfies a suitable decay condition.
    Let $\sigma$ be
    \begin{equation*}
        \sigma(x) = C\, \widetilde{\sigma}(x) \text{Osc}(x), \quad x \in \mathbb{R}^d
    \end{equation*}
    where \(\widetilde{\sigma}\in L^2(\mathbb{R}^d)\) be a twice-differentiable function that decays as described in (\ref{eqn: activation condition for AK}).
    If Osc\(:\mathbb{R}^d \to \mathbb{R}\) is a sign-changing function that is twice differentiable, uniformly bounded up to its second derivatives, and chosen to satisfy the symmetry condition \(\sigma(-x)=\sigma(x)\), then \(\sigma\) satisfies the averaging kernel conditions.
\end{example}

Under the conditions given in Proposition~\ref{thm:averaging}, we now invoke Theorems~\ref{thm: wavelet frame approximation in homogeneous type space} and \ref{thm: OGA}, where the first affirms wavelet frames and the second provides an error estimate for the corresponding approximation.
By applying these theorems to the kernels $\{S_k\}$ defined in \eqref{def:Sk}, we derive the following theorem.

\begin{theorem}\label{thm:wave:error}

Let \(\sigma:\RR^d\rightarrow \RR\) be a function satisfying the conditions in Proposition~\ref{thm:averaging}, and \(\{S_k\}_{k\in\mathbb{Z}}\) be the corresponding family of averaging kernels, i.e., \(S_k(x,b)= 2^k \sigma(2^{k/d}(x-b))\). Then \(\{\psi_{k,b}\}_{(k,b)\in \ZZ \times \R^d}\) defined as in \eqref{def:psi_kb}  constitutes a wavelet system. Furthermore, we have the following results.

\begin{enumerate}
    \item [\rm{(i)}] The wavelet system \(X(\psi)=\{\psi_{k,b}\}_{(k,b)\in\Lambda}\) derived from the averaging kernels $\{S_k\}$ admits its dual wavelet system \(X(\widetilde{\psi})=\{\widetilde{\psi}_{k,b}\}_{(k,b)\in\Lambda}\), so that, for all~\(f\in L^2(\mathbb{R}^d)\),
     \begin{equation*}
         f = \sum_{(k,b)\in \Lambda} \langle f, \widetilde{\psi}_{k,b} \rangle\, \psi_{k,b}
     \end{equation*}
    where \(\Lambda = \{(k,b) \in \mathbb{Z} \times \mathbb{R}^d: b \in 2^{-k/d} \mathbb{Z}^d\}\). Hence, \(X(\psi)\) and \(X(\widetilde{\psi})\) are wavelet frames.
    \item [\rm{(ii)}] Moreover, for every \(f \in \mathcal{L}_1\) and every \(N \in\mathbb{N}\), there exists a function
    \begin{equation*}
        f_N \in \text{span}_N\left(X(\psi)\right) \subseteq \mathcal{L}_1,
    \end{equation*}
    where span\(_N(X(\psi))\) denotes a collection of linear combinations of the wavelet system \(X(\psi)\) of at most \(N\) terms, such that 
    \begin{equation}
        \left\lVert f - f_N \rVert_{L^2} \leq \lVert f \right\rVert_{\mathcal{L}_1} (N+1)^{-1/2}.
        \label{eqn: wavelet frame approximation error}
    \end{equation}
\end{enumerate}
    \label{theorem: wavelet expansion}
\end{theorem}

Finally, considering the definitions of $S_k$ and $\psi_{k,b}$ derived from $\sigma$, we can connect the wavelet approximation \(f_N\) (via the wavelet system $X(\psi)$) to a neural network \(\Psi_{\text{W}\vec{B}}\) that uses~$\sigma$ as its activation function. In particular, the preceding wavelet approximation and error estimation results can be understood in the context of a neural network framework. With this understanding, we now prove Theorem~\ref{corollary: smooth WNN approximation}.

\begin{proof}[Proof of Theorem~\ref{corollary: smooth WNN approximation}]
    By Theorem~\ref{theorem: wavelet expansion}, for given \(f\in\mathcal{L}_1\) and \(N \in \mathbb{N}\), there exists \(f_N \in \text{span}_N(X(\psi))\) satisfying \eqref{eqn: wavelet frame approximation error}. Here, \(f_N\) can be written by
    \begin{equation}
        f_N(\vec{x}) = \sum_{k,b} \chi_{N}(k,b) c_{k,b} \psi_{k,b}(\vec{x}),
        \label{eqn:f_N}
    \end{equation}
    where $\chi_N(k,b)$ equals $1$ for at most $N$ terms and is zero otherwise.
    Then we have
    \begin{eqnarray*}
        f_N(\vec{x}) & = & \sum_{k,b} \chi_{N}(k,b) c_{k,b} 2^{-k/2} (S_k(x,b) - S_{k-1}(x,b))\\
        & = & \sum_{k,b} \chi_{N}(k,b) c_{k,b} 2^{-k/2} \left(2^k\sigma(2^{k/d}(x-b))- 2^{k-1} \sigma(2^{(k-1)/d}(x-b))\right)\\
        & = & \sum_{k,b} \chi_{N}(k,b) c_{k,b} 2^{k/2} \sigma(2^{k/d}(x-b)) - \sum_{k,b} \chi_{N}(k,b) c_{k,b} 2^{k/2-1} \sigma(2^{(k-1)/d} (x-b))\\
        & = & \Psi_{\text{W}\vec{B}}\left[\boldp\right](\vec{x})
    \end{eqnarray*}
for some parameter set $\boldp$. 
\end{proof}

\begin{figure}[t]
    \centering
    \includegraphics[width=\linewidth]{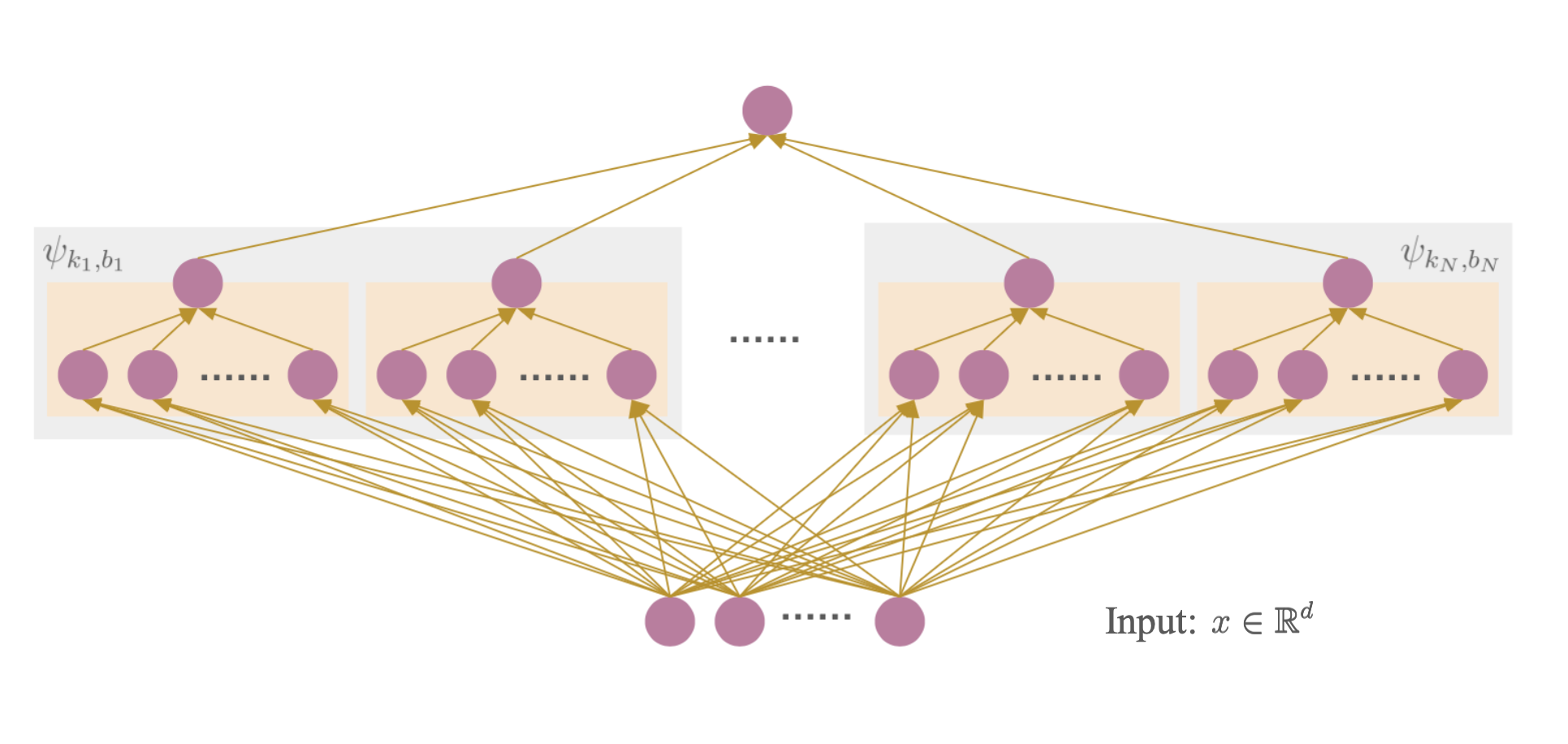}
    \caption{An architectural sketch of our network in Theorem \ref{corollary: smooth WNN approximation}.}
    \label{fig:ourNN}
\end{figure}

Figure~\ref{fig:ourNN} shows the network architecture for $\Psi_{\text{W}\vec{B}}\left[\boldp\right]$. In this W$\vec{B}$-Net, the activation function \(\sigma=\sigma_{d\to1}\) maps $\RR^d$ to $\RR$, changing the dimensionality--and thus the node count--before and after the mapping; these transitions are highlighted in orange boxes. We consider the first affine layer and the activation function together as a single hidden layer. Observe that each wavelet term $\psi_{k,b}$ can be written as
\begin{equation*}
    \psi_{k,b}(x) = 2^{k/2}\sigma(2^{k/d}(x-b)) - 2^{k/2-1} \sigma(2^{(k-1)/d} (x-b)),
\end{equation*}
which is essentially a linear combination of two activation terms. In the figure, each $\psi_{k,b}$ is represented as a gray box and corresponds to two nodes in the hidden layer. Consequently, when $\Psi_{\text{W}\vec{B}}\left[\boldp\right]$ incorporate $N$ wavelet terms, the resulting network has $2N$ nodes in its hidden layer.  

In Appendix~\ref{Appendix}, we compare our network with those in \cite{Frischauf:2024:QNN,Shaham:2018:PAP}, which also use the wavelet-based framework for neural network approximation.

\subsection{Proof of Proposition \ref{thm:averaging}}
\label{subsec: proof}
Recall that we set $\rho(x,b)=c\|x-b\|^d$ for $x,d\in\RR^d$. We begin with a lemma that will be useful in proving Proposition~\ref{thm:averaging}.
\begin{lemma}
Fix $k\in\NN$. For the triple \((x, x', b)\) satisfying 
    \begin{equation*}
        \rho(x,x') \leq 2^{-d}\big(2^{-k}+\rho(x,b)\big),
    \end{equation*}
it holds that for any \(z\) between \(x\) and \(x'\),
    \begin{equation*}
        \lVert z - b \rVert^d \geq 2^{-d} \big( \lVert x-b \rVert^d - c^{-1}2^{-k}\big).
    \end{equation*}
    Here, `$z$ between $x$ and $x'$' indicates that $z$ on the open line segment connecting $x$ and $x'$. 
    
    Furthermore, for the quadruple \((x,x',b,b')\) satisfying 
    \begin{equation*}
        \rho(x,x') \leq 3^{-d}\big(2^{-k} + \rho(x,b)\big)\quad \text{and} \quad \rho(b,b') \leq 3^{-d} \big(2^{-k} + \rho(x,b)\big),
    \end{equation*}
    it holds that for any \(z\) between \(b\) and \(b'\), and \(z'\) between \(x\) and \(x'\),
    \begin{equation*}
        \lVert z'-z \rVert^d \geq 3^{-d} \big( \lVert x-b \rVert^d - c^{-1} 2^{1-k} \big).
    \end{equation*}
    \label{lemma: distance btw z-z'}
\end{lemma}

\begin{proof}
Since the two inequalities can be proved using the same argument, we solely present the proof of the second. To show the second inequality, we first recall Jensen's inequality: for $u,v,w\in\RR^+$,
$\lambda u^d + \lambda v^d +\lambda  w^d 
\geq (\lambda u+\lambda v +\lambda w)^d$ with $\lambda=1/3.$
This implies that \[u^d + v^d + w^d \geq 3^{1-d}(u+v+w)^d\quad\mbox{for }u,v,w \geq 0.\]
Set \(z' = x + \tilde{t}(x'-x)\) and \(z = b+ t(b'-b)\) for \(0 \leq t, \tilde{t} \leq 1\). 
By employing Jensen's inequality and the triangle inequality, we obtain that
    \begin{align*}
        \lVert z'-z \rVert^d 
        &=\lVert x + \tilde{t}(x'-x) - b - t(b'-b) \rVert^d \\
        & \geq  3^{1-d} \lVert x-b \rVert^d - \tilde{t}^d \lVert x'-x \rVert^d - t^d \lVert b'-b \rVert^d\\
        & \geq  3^{1-d} \lVert x-b \rVert^d - \lVert x'-x \rVert^d - \lVert b'-b \rVert^d\\
        & \geq  3^{1-d} \lVert x-b \rVert^d - 2 \cdot 3^{-d}c^{-1} \big(2^{-k} + c \lVert x-b \rVert^d\big)\\
        & =  3^{-d} \big(\lVert x-b \rVert^d - c^{-1}2^{1-k}\big).
    \end{align*}
\end{proof}

\begin{proof}[Proof of Proposition~\ref{thm:averaging}]
We now verify each of the conditions \eqnref{eqn: averaging kernel3}--\eqref{eqn: averaging kernel4} in Definition~\ref{def:averaging} one by one. 
In this proof, in accordance with Definition \ref{def:Sk}, we use $b$ for the variable of the second component of $S_k(\cdot,\cdot)$.

\smallskip

{\bf (\ref{eqn: averaging kernel3}).} One can easily find that 
    \begin{equation*}
        \int_{\mathbb{R}^d} S_k(x,b) db = \int_{\mathbb{R}^d} 2^k \varphi(2^{k/d} (x-b)) db = \int_{\mathbb{R}^d} \varphi(x) dx = 1.
    \end{equation*}

\smallskip

{\bf (\ref{eqn: averaging kernel1}).} We observe from \eqref{eqn: activation condition for AK} that
    \begin{equation*}
        \lvert \sigma(x) \rvert \leq \frac{C'}{\left(c^{-1} + \lVert x \rVert^d\right)^{1+\epsilon}}
    \end{equation*}
 and, thus,
    \begin{equation*}
        \lvert S_k(x,b) \rvert = 2^k \left\lvert \sigma (2^{k/d} (x-b)) \right\rvert \leq \frac{2^{-k\epsilon} c^{1+\epsilon}  C'}{\left(2^{-k} + c \lVert x-b \rVert^d\right)^{1+\epsilon}}.
    \end{equation*}
 Therefore, (\ref{eqn: averaging kernel1}) holds by setting \(C = c^{1+\epsilon} C'\).

\medskip

 {\bf (\ref{eqn: averaging kernel2}).} By the mean value theorem, it holds that
    \begin{equation*}
        \frac{\left\lvert S_k(x,b) - S_k(x',b)\right\rvert}{\rho(x,x')^{1/d}} 
        \leq \frac{1}{c^{1/d}} \sup_{z\text{ between }x, x'} \left\lVert \nabla_x S_k(z,b) \right\rVert.
    \end{equation*} 
We observe from \eqref{eqn: activation condition for AK} that
    \begin{equation}\notag
        \left\lVert \nabla_x \sigma(x) \right\rVert \leq \frac{C'}{\left(c^{-1} + \lVert x \rVert^d\right)^{1 + \epsilon + 1/d}}
    \end{equation}
and, hence, 
    \begin{align*}
        \lVert \nabla_x S_k(z,b) \rVert & =  2^{k(1+1/d)} \big\lVert \nabla_x\, \sigma(2^{k/d}(z-b)) \big\rVert\\
        & \leq  C'\, 2^{k(1+1/d)}\, \frac{1}{\left(c^{-1}+2^k \lVert z-b \rVert^d\right)^{1+\epsilon + 1/d}}.
  \end{align*}

Assume that $\rho(x, x') \leq 3^{-d}\left(2^{-k} + \rho(x, y)\right)$. 
Then, by Lemma~\ref{lemma: distance btw z-z'} and the fact that \(1-2^{-d} \geq 2^{-d}\), we observe
    \begin{equation*}
        c^{-1} + {2^k} \lVert z-b \rVert^d \geq c^{-1}(1-2^{-d}) + 2^{k}2^{-d} \lVert x - b \rVert^d \geq 2^{-d} \left(c^{-1} + 2^k \lVert x-b \rVert^d\right). 
    \end{equation*}

One can then obtain an upper bound of \(\lVert \nabla_x S_k(z,b) \rVert\) as
    \begin{align*}
    \left\lVert \nabla_x S_k(z,b) \right\rVert 
    & \leq  C'\,2^{k(1+1/d)} \,\frac{2^{1+d(1+\epsilon)}}{\left(c^{-1} + 2^k \lVert x-b \rVert^d\right)^{1+\epsilon+1/d}}\\
    & =  2^{1+d(1+\epsilon)}\, c^{1+\epsilon+1/d}\, C'\, \frac{1}{\left(2^{-k}+c \lVert x-b \rVert^d\right)^{1/d}} \frac{2^{-k\epsilon}}{\left(2^{-k}+c \lVert x-b \rVert^d\right)^{1+\epsilon}}.
    \end{align*}
 Therefore, (\ref{eqn: averaging kernel2}) holds with \(\eta = 1/d\) and \(C= 2^{1+d(1+\epsilon)} c^{1+\epsilon} C'\).

\medskip

{\bf (\ref{eqn: averaging kernel4}).} We first see that
    \begin{equation*}
        \frac{\left\lvert S_k(x,b) - S_k(x',b) - S_k(x,b') + S_k(x',b') \right\rvert}{\rho(x,x')^{1/d} \rho(b,b')^{1/d}} \leq \frac{1}{c^{2/d}} \frac{\left\lvert F(b) - F(b') \right\rvert}{\lVert b-b' \rVert}
    \end{equation*}
with 
    \begin{equation}\label{def:F}
        F(\cdot) := \frac{S_k(x,\cdot) - S_k(x',\cdot)}{\lVert x- x' \rVert}.
    \end{equation}
By the mean value theorem, we have 
    \begin{equation*}
        \frac{\left\lvert F(b) - F(b') \right\rvert}{\lVert b-b' \rVert} \leq \sup_{z \text{ between } b, b'} \left\lVert \nabla F(z) \right\rVert.
    \end{equation*}
    By again applying the mean value theorem to \eqref{def:F}, we finally obtain
    \begin{equation*}
        \frac{\left\lvert S_k(x,b) - S_k(x',b) - S_k(x,b') + S_k(x',b') \right\rvert}{\rho(x,x')^{1/d} \rho(b,b')^{1/d}} \leq \frac{1}{c^{2/d}}\sup_{z \text{ between } b, b'} \sup_{z' \text{ between } x, x'}\left\lVert\nabla^2_{x,b}\, S_k(z',z) \right\rVert. 
    \end{equation*}
    We observe from \eqref{eqn: activation condition for AK} that
    \begin{equation*}
        \big\lVert \nabla^2_{x,b} \, \sigma(x-b) \big\rVert \leq C' \frac{1}{\left(c^{-1} + \lVert x-b \rVert^d\right)^{1+\epsilon +2/d}}.
        \label{eqn: sufficient condition for general phi}
    \end{equation*}
    From this, we see that
    \begin{align*}
        \left\lVert \nabla^2_{x,b}\, S_k(z',z)\right \rVert 
        & =  2^{k(1+2/d)} \big\lVert \left(\nabla^2_{x,b}\, \sigma(x-b)\right)\big\vert_{x=2^{k/d}z', \,b = 2^{k/d}z}\big\rVert
        \\
        & \leq  C' 2^{k(1+2/d)} \, \frac{1}{\left(c^{-1} + 2^k \lVert z'-z \rVert^d\right)^{1+\epsilon + 2/d}}.
    \end{align*}

Assume that $\rho(x, x') \leq 3^{-d}\left(2^{-k} + \rho(x, b)\right)$ and $\rho(b, b') \leq 3^{-d}\left(2^{-k} + \rho(x, b)\right)$. Then, by Lemma~\ref{lemma: distance btw z-z'} and the the fact that \(1-2\cdot 3^{-d} \geq 3^{-d}\), we have
    \begin{equation*}
        c^{-1} + 2^{k} \lVert z' - z \rVert^d \geq c^{-1}(1-2 \cdot 3^{-d}) + 2^k 3^{-d} \lVert x - b \rVert^d \geq 3^{-d} \big(c^{-1} + 2^k \lVert x - b \rVert^d\big)
    \end{equation*}
for $z$ between $x,x'$, and $z'$ between $b,b'$. 
    Thus, we can bound \(\lVert \nabla^2_{x,b} S_k(z',z) \rVert\) as follows:
    \begin{align*}
        \left\lVert \nabla^2_{x,b}\, S_k(z',z) \right\rVert 
        & \leq  C' 2^{k(1+2/d)} \frac{3^{2 + d(1+\epsilon)}}{\left(c^{-1}+2^k \lVert x-b \rVert^d\right)^{1+\epsilon+2/d}}\\
        & =  3^{2+d(1+\epsilon)} c^{1+\epsilon+2/d}\, C' \,\frac{1}{\left(2^{-k}+ c \lVert x-b \rVert^d\right)^{2/d}} \frac{2^{-k\epsilon}}{\left(2^{-k}+ c \lVert x-b \rVert^d\right)^{1+\epsilon}}.
    \end{align*}
    Therefore \(\{S_k\}_{k\in\mathbb{Z}}\) satisfies the double Lipschitz condition, (\ref{eqn: averaging kernel4}), with \(\eta = 1/d\) and \(C = 3^{2+d(1+\epsilon)} c^{1+\epsilon} C'\).
\end{proof}

\section{Neural networks with non-smooth activation functions}\label{sec:general}

Until now, we have developed our approximation theory under the assumption that neural network activation functions are twice-differentiable. In this section, we build on those results to relax the smoothness requirement, thereby extending the theory to encompass more general activation functions, notably non-smooth ones. We use the notation
$$\operatorname{dist}(\sigma_1,\sigma_2):= \left(1+\frac{1}{\sqrt{2}}\right)\left\|\sigma_1-\sigma_2\right\|_{L^2},\quad \sigma_1,\sigma_2\in L^2(\RR^d).$$

Let $\sigma:\R^d \to \R$ be a twice-differentiable function satisfying the conditions in Proposition~\ref{thm:averaging}. We also adopt the settings of Theorems~\ref{corollary: smooth WNN approximation} and \ref{theorem: wavelet expansion}.
In particular, we employ the space $\mathcal{L}_1$ which is defined as in \eqnref{def:L1} with the discrete wavelet frame \(X(\psi)\) constructed in Theorem~\ref{theorem: wavelet expansion}.

Let $f\in \mathcal{L}_1$. Consider a sequence of approximations $f_N$, each formed as a linear combination of at most $N$ wavelet elements in $\{\psi_{k,b}\}_{(k,b)\in \Lambda}$, that satisfies the convergence criterion in \eqref{ineq:f_N:OGA}. Specifically,
\beq\label{f_N:expression}
f_N(\vec{x})=\sum_{(k,b)\in\Lambda} {\chi}_N (k,b) c_{k,b}\psi_{k,b}(\vec{x}), \quad\vec{x}\in\RR^d,
\eeq
where $\chi_N(k,b)$ equals $1$ for at most $N$ terms and is zero otherwise. 

\begin{theorem}\label{lem:disconti}
Let \(\sigma:\R^d \to \R\) be a twice-differentiable function satisfying conditions in Proposition~\ref{thm:averaging} and \(\mathcal{L}^1\) be its associated space. For any function $\sigma^{\dagger}\in L^2(\RR^d)$, let $\psi^\dagger_{k,b}$ be a collection of functions defined as in Definition~\ref{def:Sk and psi_kb} with $\sigma^\dagger$ in the place of $\sigma$.
For $f\in\mathcal{L}_1$ and $N\in\NN$, let $f_N$ be given as in \eqnref{f_N:expression}, and set \begin{equation}\label{def:fN_dagger}
f^\dagger_N:=\sum_{(k,b)\in\Lambda} {\chi}_N (k,b) c_{k,b}\psi^\dagger_{k,b}.
\end{equation}
Then we have
\begin{equation*}
\big\|f-f_N^\dagger\big\|_{L^2} \leq \|f\|_{\mathcal{L}_1}(N+1)^{-1/2}+ \operatorname{dist}(\sigma,\sigma^\dagger) \sum_{(k,b)\in \Lambda}\chi_N(k,b) \left|c_{k,b}\right|.
\end{equation*}
\end{theorem}
\begin{proof}
We first observe that 
    \begin{align*}
        \big\lVert \psi_{k,b} -  \psi_{k,b}^{\dagger} \big\rVert_{L^2} 
        =  &\ \Big\lVert 2^{k/2} \sigma(2^{k/d}(x-b)) - 2^{k/2-1} \sigma(2^{(k-1)/d}(x-b)) \\
          & \qquad- \left( 2^{k/2} \sigma^{\dagger}(2^{k/d}(x-b)) - 2^{k/2-1} \sigma^{\dagger}(2^{(k-1)/d}(x-b)) \right) \Big\rVert_{L^2} \\
         \leq&\  \Big\lVert 2^{k/2} \left(\sigma(2^{k/d}(x-b)) - \sigma^{\dagger}(2^{k/d}(x-b))\right)\Big\rVert_{L^2} \\
        & \qquad  + \Big\lVert 2^{k/2-1} \left( \sigma(2^{(k-1)/d}(x-b)) -  \sigma^{\dagger}(2^{(k-1)/d}(x-b)) \right)\Big\rVert_{L^2}.
        \end{align*}
By changing the variables in the integrals, we have
$$   \big\lVert \psi_{k,b} -  \psi_{k,b}^{\dagger} \big\rVert_{L^2}
        \leq \big\lVert \sigma -  \sigma^{\dagger} \big\rVert_{L^2} + \frac{1}{\sqrt{2}} \big\lVert \sigma - \sigma^{\dagger} \big\rVert_{L^2}.
$$
It then follows that
\begin{align*}
\big\|f-f_N^\dagger\big\|_{L^2} 
&\leq \big\|f-f_N\big\|_{L^2}+\big\|f_N-f_N^\dagger\big\|_{L^2}\\
&\leq\|f\|_{\mathcal{L}_1}(N+1)^{-1/2}+ 
 \bigg\lVert \sum_{k,b} \chi_{N} c_{k,b} \left(\psi_{k,b} - \psi^{\dagger}_{k,b}\right) \bigg\rVert_{L^2}.
\end{align*}
This proves the desired result. 
\end{proof}

By the same arguments of the proof of Theorem~\ref{corollary: smooth WNN approximation} in Section~\ref{subsec:main}, a function $f_N^\dagger$ given as in~\eqref{def:fN_dagger} can be expressed as the W$\vec{B}$-Net in (\ref{eqn: SW-VB Net}) with activation function $\sigma^\dagger$.
Indeed, it is easy to see that 
    \begin{align*}
     f^\dagger_N & =   \sum_{k,b} \chi_{N}(k,b) c_{k,b} \psi^{\dagger}_{k,b}\\
   & =  \sum_{k,b} \chi_{N}(k,b) c_{k,b} 2^{-k/2} \left( 2^k \sigma^{\dagger}(2^{k/d}(x-b)) - 2^{k-1} \sigma^{\dagger}(2^{(k-1)/d} (x-b))\right)\\
        & =  \sum_{k,b} \chi_{N}(k,b) c_{k,b} 2^{k/2} \sigma^{\dagger}(2^{k/d}(x-b))  -  \sum_{k,b} \chi_{N}(k,b) c_{k,b} 2^{k/2-1} \sigma^{\dagger}(2^{(k-1)/d} (x-b)).
    \end{align*}

Specifically, we have
\begin{equation*}
    f_N^\dagger(\vec{x}) =  \Psi_{\text{W}\vec{B}}\big[\boldp;\sigma^\dagger\big](\vec{x}),
\end{equation*}
where 
\begin{equation}
    \Psi_{\text{W}\vec{B}}\big[\boldp;\sigma^\dagger\big](\vec{x}) =  \sum_{n=1}^{2N} \alpha_{n} \sigma^{\dagger}(\gamma_{n} \vec{x} + \vec{\theta}_{n})
    \label{Psi:sigma_dagger}
\end{equation}
with a suitably chosen parameter set 
\begin{equation}\label{parameter:coro}
\boldp = \big[\gamma_1, \cdots, \gamma_{2N};\, \alpha_1, \cdots, \alpha_{2N};\, \vec{\theta}_1, \cdots, \vec{\theta}_{2N}\big]\in \RR^{2N}\times\RR^{2N}\times\RR^{d\times 2N}.
\end{equation}
Consequently, we obtain the following corollary directly from Theorem \ref{lem:disconti}. 

\begin{corollary}\label{cor:main:convergence}
Let $\sigma$, $\sigma^\dagger$ and $\mathcal{L}_1$ be as in Theorem~\ref{lem:disconti}. 
For every \(f\in \mathcal{L}_1\) and \(N\in\mathbb{N}\), there exists a parameter set $\boldp$ of the form \eqnref{parameter:coro} such that
  \begin{equation}
        \Big\lVert f-\Psi_{\text{W}\vec{B}}\big[\boldp;\sigma^\dagger\big] \Big\rVert_{L^2} 
       \leq \|f\|_{\mathcal{L}_1}(N+1)^{-1/2}+ \operatorname{dist}(\sigma,\sigma^\dagger)\sum_{(k,b)\in \Lambda}\chi_N(k,b) \left|c_{k,b}\right|,
       \label{eqn: network approximation with non-smooth ftn}
    \end{equation}  
where $\Psi_{\text{W}\vec{B}}\big[\boldp;\sigma^\dagger\big]$ is the W$\vec{B}$-Net with activation function $\sigma^\dagger$ given by \eqref{Psi:sigma_dagger}.
\end{corollary}

A natural approach to reducing the distance between \(\sigma\) and \(\sigma^\dagger\) is to express \(\sigma^\dagger\) as a linear combination of translations and scalar multiplications of a single function \(\sigma_0\) (see \eqref{eqn:smooth_activation_with_M}), and to increase the number of terms \(M\) in the linear combination, as in the typical method of approximating functions by step functions. It is noteworthy that the neural network defined with~\(\sigma^\dagger\) retains the structure of a W$\vec{B}$-Net, even as \(M\) increases (refer to \eqref{eqn: non-smooth NN}). Using this approach, we now present a practical way to control the approximation error while preserving a unified neural network structure.

\begin{theorem}\label{theorem:final}
Let $\sigma_0\in L^2(\RR^d)$ be arbitrary. Also, let $\sigma:\R^d \to \R$ be a twice-differentiable function satisfying the conditions in Proposition~\ref{thm:averaging}, and \(\mathcal{L}_1\) be its associated space, as in Theorem~\ref{lem:disconti}. For $\epsilon>0$, assume that
\begin{equation}
    \operatorname{dist}(\sigma,\sigma^\dagger)<\epsilon, \quad \text{where }\sigma^\dagger:=\sum_{m=1}^M c_m \sigma_0 (x-b_m)
\label{eqn:smooth_activation_with_M}
\end{equation}
with some fixed $M=M(\epsilon)\in\NN$, $b_m\in\RR^d$ and $c_m\in\RR$.
Then for $f\in\mathcal{L}_1$ and \(N\in\mathbb{N}\), \(f\) can be approximated by a neural network with the activation function $\sigma_0$ as 
\begin{equation}
    \big\lVert f-\Psi_N \big\rVert_{L^2} 
       \leq \|f\|_{\mathcal{L}_1}(N+1)^{-1/2}+ \epsilon\sum_{(k,b)\in \Lambda}\chi_N(k,b) \left|c_{k,b}\right|,
    \label{eqn: controlled network approximation with nonsmooth ftns}
\end{equation}
where
\begin{equation}
    \Psi_N(\vec{x}) = \sum_{n=1}^{2N\cdot M} \alpha_{n} \sigma_0 (\gamma_{n}\vec{x} + \vec{\theta}_n),
        \label{eqn: non-smooth NN}
\end{equation}
with some parameters $\alpha_n$, $\gamma_n$ and $\vec{\theta}_n$. Here, $\chi_N(k,b)$ is given as in \eqnref{f_N:expression}.
\end{theorem}

\begin{proof}
    Under the assumptions, we obtain 
    \begin{equation}
        \Big\lVert f-\Psi_{\text{W}\vec{B}}\big[\boldp;\sigma^\dagger\big] \Big\rVert_{L^2} 
       \leq \|f\|_{\mathcal{L}_1}(N+1)^{-1/2}+ \operatorname{dist}(\sigma,\sigma^\dagger)\sum_{(k,b)\in \Lambda}\chi_N(k,b) \left|c_{k,b}\right|,
       \label{eqn: last eqn}
    \end{equation}
    by invoking Corollary~\ref{cor:main:convergence}. Then, substituting the expression of \(\sigma^\dagger\), as given in~\eqref{eqn:smooth_activation_with_M}, into \eqref{Psi:sigma_dagger}, and applying the bound \(\epsilon\) on the distance in~\eqref{eqn: network approximation with non-smooth ftn}, we complete the proof.
\end{proof}

\begin{remark}
    From \eqref{eqn: last eqn} {\rm{(}}see also {\rm{\eqref{eqn: controlled network approximation with nonsmooth ftns})}}, decreasing the $L^2$-distance between $\sigma$ and $\sigma^\dagger$, governed by $\epsilon$, yields a tighter approximation bound for the W$\vec{B}$-Net \(\Psi_N\). To achieve the smaller~$L^2$-distance, one may increase~$M$, which in turn raises the number of network nodes, as indicated in \eqref{eqn: non-smooth NN}.
\end{remark}

\section{Conclusion}\label{sec:conclusion}

In this paper, we have developed a neural network approximation using a wavelet-based framework that is based on wavelet frame theory on spaces of homogeneous type. While previous wavelet-based approaches have often restricted the class of activation functions for specific research needs, our work extends their applicability by introducing sufficient conditions for a wider range of activation functions. In particular, these conditions accommodate various function classes, including those that are twice differentiable, potentially oscillatory, provided they exhibit suitable decay conditions. 

Nevertheless, the conditions still require the activation functions to meet a certain degree of smoothness due to the double Lipschitz condition in the wavelet frame theory on spaces of homogeneous type. To address this limitation and cover piecewise-smooth activation functions that are not necessarily twice differentiable, we propose to use the \(L^2\)-distance between smooth and non-smooth activation functions. We demonstrate that non-smooth activation functions that are close to smooth functions, with respect to the above distance, can still yield neural network approximations with a controllable error bound. In particular, this distance can be reduced by increasing the number of network nodes. Overall, we establish a theoretical foundation for ensuring convergence in wavelet-based neural network approximation across a broader and more practical class of activation functions.

\begin{appendices}

\section{Connection between $\vecW$B-Net and W$\vecB$-Net}
\label{appx:connection between neural networks}

We present the connection between $\vecW$B-Net and W$\vecB$-Net. We show, in particular, that $W\vecB$-Net is a special case of $\vecW B$-Net under the conditions specified below. Consider $\Psi_{\text{W}\vec{B}}$ in~\eqref{eqn: SW-VB Net}, which uses a vector-to-scalar activation $\sigma_{d\to1}$ of the form
\beq\label{eqn:cond:WB}
\sigma_{d\to1}(\vec{x})=\big(\sigma_{1\to1}\circ\vecone\big)(\vec{x}),\quad \vec{x}\in\RR^d,
\eeq
where
$\sigma_{1\to 1}$ is a scalar-to-scalar function and
 \(\vecone = (1,\cdots,1) \in \mathbb{R}^d\) acts on $\vec{x}\in\RR^d$ by $\vecone(\vec{x})=\vecone\cdot\vec{x}$. Then our $\Psi_{\text{W}\vec{B}}$ can be represented by the usual $\vec{W}B$-Net in \eqref{eqn: VW-SB Net} with the activation $\sigma_{\vec{W}B}=\sigma_{1\to 1}$.
 
 More precisely, let \(\boldp=[(\gamma_1, \cdots, \gamma_N);(\alpha_1, \cdots, \alpha_N);[\vec{\theta}_1\, \cdots\, \vec{\theta}_N]] \in \mathbb{R}^N \times \mathbb{R}^N \times \mathbb{R}^{d \times N}\) be a parameter set for the W$\vecB$-Net in \eqref{eqn: SW-VB Net} that satisfies \eqref{eqn:cond:WB}. Then we have
 \begin{equation*}
  \Psi_{\text{W}\vec{B}}\left[\boldp\right](\vec{x}) = \sum_{n=1}^N \alpha_n \big(\sigma_{\vec{W}B} \circ \vecone\big)(\gamma_n \vec{x} + \vec{\theta}_n) =\sum_{n=1}^N \alpha_n\,\sigma_{\vec{W}B} \big(\gamma_n\vecone\cdot \vec{x} + \vecone\cdot\vec{\theta}_n\big) = \Psi_{\vec{W}B}\left[\boldp'\right](\vec{x}),
\end{equation*}
where $\boldp'=[W;\vec{\alpha};\vec{\beta}]$, $W=[\gamma_1\vecone\,\cdots\,\gamma_N\vecone]$, $\vec{\alpha}=(\alpha_1,\cdots,\alpha_N)$, and $\vec{\beta}=(\vecone\cdot\vec{\theta}_1,\cdots,\vecone\cdot\vec{\theta}_N)$. In other words, from an expressiveness standpoint, any function represented by a W$\vecB$-Net that uses an activation of the form \eqref{eqn:cond:WB} naturally lies in the function space of the $\vecW$B-Net.  

\section{Architectural overview of related work on wavelet-based neural approximations}
    \label{Appendix}

We present two key results from the literature on wavelet-based neural approximations \cite{Frischauf:2024:QNN, Shaham:2018:PAP}.

The first uses the radial quadratic neural networks (RQNNs) \cite{Frischauf:2024:QNN}. 
The architecture of RQNNs shares a similar structure with ours, but it uses an activation function $\widetilde{\sigma}: \mathbb{R}\rightarrow\mathbb{R}$ in composition with a radial quadratic form from $\RR^d\to\RR$, resulting in a function $\sigma:\RR^d\to\RR$ given by
\[\sigma(\cdot)=\widetilde{\sigma}(r^2 - \lVert \cdot \rVert^2)\quad\mbox{for some fixed constant }r>0.\]
Then wavelet system \(\psi_{k,b}\) and \(S_k\) are defined by \eqref{def:Sk} and \eqref{def:psi_kb} with \(\sigma\), that is, 
    \begin{eqnarray*}
        S_k(x,b) & = &  2^k \sigma\left(2^{k/d}(x-b)\right)\\
        & = & 2^k \widetilde{\sigma}\left(r^2 - \lVert 2^{k/d}(x - b) \rVert^2 \right)\\
        \psi_{k,b}(x) & = & 2^{k/2}\sigma\left(2^{k/d}(x-b)\right) - 2^{k/2-1} \sigma \left(2^{(k-1)/d} (x-b) \right)\\
        & = & 2^{k/2}\widetilde{\sigma}\left(r^2 - \lVert 2^{k/d}(x-b) \rVert^2 \right) - 2^{k/2-1} \widetilde{\sigma} \left(r^2 - \lVert 2^{(k-1)/d} (x-b)\rVert^2 \right).
    \end{eqnarray*}

Figure~\ref{fig:RQNN} shows the RQNN architecture (for comparison, see Figure~\ref{fig:ourNN}). The primary difference between our networks and RQNNs lies in the layer immediately following the input; while we use an affine layer, RQNNs utilize a radial quadratic layer. This structural difference affects the domain of the activation function (from $\RR^d$ to $\RR$ in ours versus from $\RR$ to $\RR$ in RQNNs)

This distinction is visually represented by the orange boxes, which display the action of activation function \(\widetilde{\sigma}\). The gray box corresponds to \(\psi_{k,b}\). Consistent with our previous interpretation in Section~\ref{sec:wavelet:approx}, we regard the first radial quadratic layer and the activation function together as a single hidden layer. From this perspective, each \(\psi_{k,b}\) corresponds to two nodes in the hidden layer; hence, a sum of \(N\) wavelet terms \(\psi_{k,b}\) can be represented by a neural network with \(2N\) nodes in the hidden layer.

    \begin{figure}[h]
        \centering
        \includegraphics[width=0.6\linewidth]{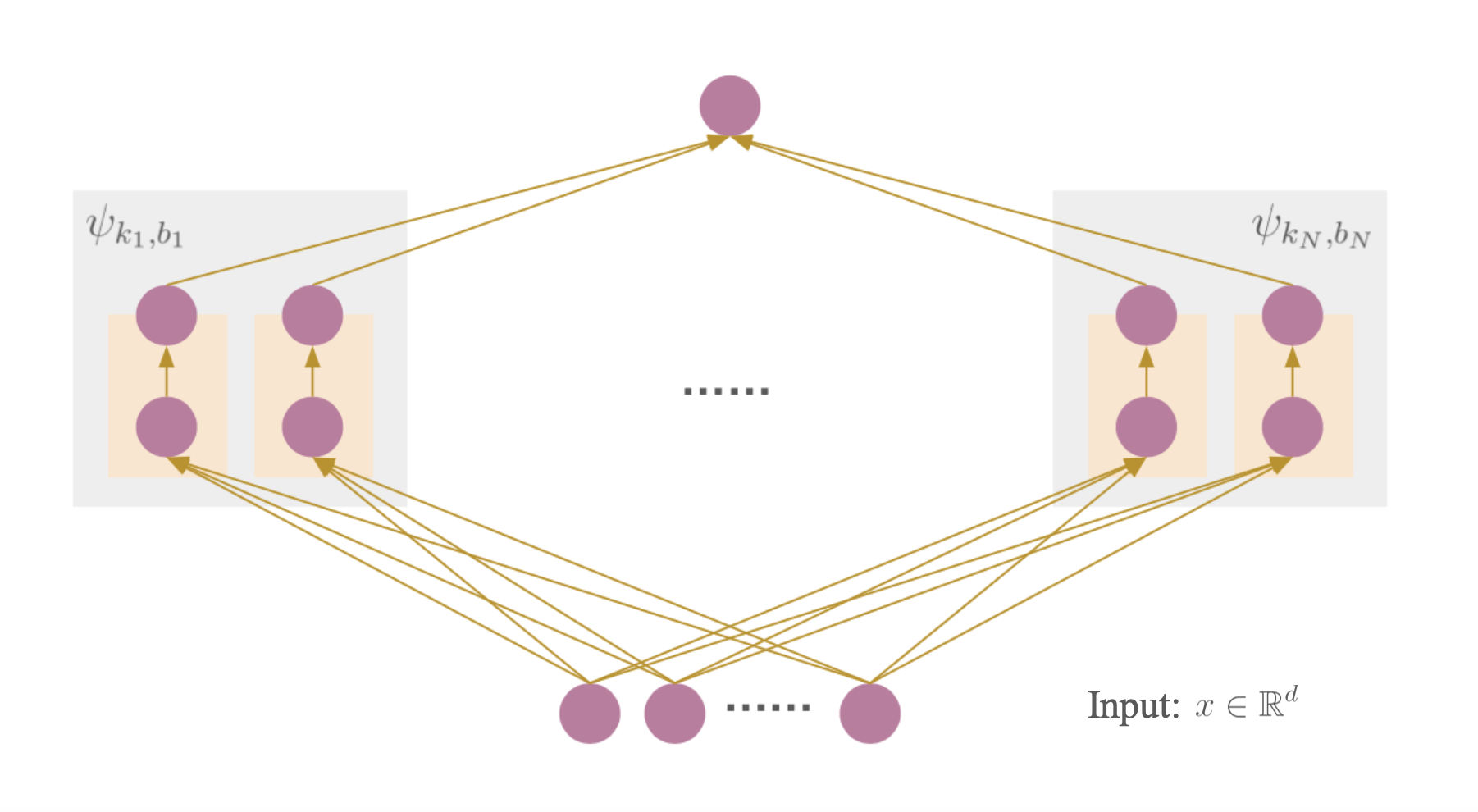}
        \caption{An architectural sketch of radial quadratic neural network in \cite{Frischauf:2024:QNN}.}
        \label{fig:RQNN}
    \end{figure}

The second example comes from \cite{Shaham:2018:PAP}, where the authors employ the ReLU activation function. 
The wavelet functions $\psi_{k,b}$ are defined by using the activation function $\sigma:\RR^d\to\RR$, given in terms of ReLU as follows:
\begin{equation*}
\sigma(x) = ReLU\left( \sum_{j=1}^d L(x_j) - 2(d-1)\right), \quad x\in\RR^d,
\end{equation*}
where \(L(x_j)= ReLU(x_j + 3) - ReLU(x_j +1) - ReLU(x_j -1) + ReLU(x_j -3)\). Consequently, \(\sigma\) can be realized by a network with 4d rectifier units in the first layer and a single unit in the second layer. For more details, including the architectural sketch, see \cite{Shaham:2018:PAP}.

Notably, the proposed network in this paper accommodates general activation functions, thereby allowing flexibility in the choice of activation functions. 

\end{appendices}


\end{document}